\documentclass{article}

\usepackage{style}
\usepackage{natbib}
\usepackage[utf8]{inputenc} 
\usepackage[T1]{fontenc}    
\usepackage{hyperref}       
\usepackage{url}            
\usepackage{booktabs}       
\usepackage{authblk}
\usepackage{amsfonts}       
\usepackage{nicefrac}       
\usepackage{microtype}      
\usepackage{dutchcal}
\usepackage{xcolor}         
\usepackage{mathtools} 
\usepackage{amssymb}
\usepackage{tikz} 
\usepackage[linesnumbered]{algorithm2e}
\RestyleAlgo{ruled}
\usepackage{amsthm}
\usepackage{cuted}
\usepackage{ amssymb }
\usepackage{amsmath}
\usepackage{stmaryrd}
\usepackage{subcaption}
\usepackage{bbm}
\usepackage[inline]{enumitem}
\usepackage[font=scriptsize,labelfont=bf]{caption}
 
\DeclareMathOperator*{\argmin}{argmin}
\usepackage{mdframed}
\newmdtheoremenv{theorem}{Theorem}
\newmdtheoremenv{corollary}{Corollary}[theorem]
\usepackage{tikz}

\usepackage{xparse} \DeclarePairedDelimiterX{\Iintv}[1]{\llbracket}{\rrbracket}{\iintvargs{#1}}
\providecommand{\keywords}[1]
{
  \small	
  \textbf{\textit{Keywords---}} #1
}

\title{Fully Probabilistic Design for Optimal Transport}
\author[1]{Sarah Boufelja Y.}
\author[2]{Anthony Quinn}
\author[3]{Martin Corless}
\author[1]{Robert Shorten}
\affil[1]{Dyson School of Design Engineering, Imperial College London}
\affil[2]{ Department of Electronic and Electrical Engineering, Trinity College Dublin}
\affil[3]{Department of Aeronautics and Astronautics, Purdue University}

\begin{document}
\maketitle
\section*{Abstract}
The goal of this paper is to introduce a new theoretical framework for Optimal Transport (OT), using the terminology and techniques of Fully Probabilistic Design (FPD). Optimal Transport is the canonical method for comparing probability measures and has been successfully applied in a wide range of areas (computer vision \cite{Rubner2004TheEM}, computer graphics \cite{article}, natural language processing \cite{pmlr-v37-kusnerb15}, etc.). However, we argue that the current OT framework suffers from two shortcomings: first, it is hard to induce generic constraints and probabilistic knowledge in the OT problem; second, the current formalism does not address the question of uncertainty in the marginals, lacking therefore the mechanisms to design robust solutions. By viewing the OT problem as the optimal design of a probability density function with marginal constraints, we prove that OT is an instance of the more generic FPD framework. In this new setting, we can furnish the OT framework with the necessary mechanisms for processing probabilistic constraints and deriving uncertainty quantifiers, hence establishing a new extended framework, called FPD-OT. Our main contribution in this paper is to establish the connection between OT and FPD, providing new theoretical insights for both. This will lay the foundations for the application of FPD-OT in a subsequent work, notably in processing more sophisticated knowledge constraints, as well as in designing robust solutions in the case of uncertain marginals. \\\\
\keywords{Optimal Transport, Fully Probabilistic Design, Convex optimization}

\section{Introduction}
Many Machine Learning (ML) problems reduce to the question of comparing and summarizing probability measures. For example, problems in domain adaptation, adversarial training and distributed learning fall withing this setting. Measuring the distance between two probability measures can generally be addressed using any divergence function \cite{article}. However, such functions often do not consider the spatial properties and the physical distances between these measures. Optimal Transport on the other hand, converts the distance between samples to a distance between probability measures, and in doing so endows the space of measures with a topology: if the underlying space is Euclidean, the concepts of interpolation, barycenters and gradient of functions are naturally extended to the space of measures \cite{https://doi.org/10.48550/arxiv.1803.00567}. 
\\\\
Notwithstanding this important property, when it comes to the practical requirement of incorporating probabilistic constraints in the OT problem, or processing uncertainty in the marginals, there exists no systematic or generic methodology. This latter observation motivates our interest in Fully Probabilistic Design (FPD). FPD is the axiomatic framework for designing probability measures under uncertainty, while being consistent with particular knowledge constraints, imposed by the designer \cite{KARNY2020104719}. These constraints express any additional information about the unknown probability measure, in the form of a prior density, a set membership, some physical laws, etc. \cite{QUINN2016532}. In this paper, we prove that regularised OT is a special case of FPD. Indeed, by formulating the OT problem as a constrained design of a joint probability density, the connection with FPD emerges naturally. We argue that this connection yields a structured and robust FPD-OT framework. \\\\
Our paper is structured as follows. We begin in Section \ref{OT} by reviewing the key concepts of OT, emphasising the mathematical objects used later in the document. In Section \ref{7}, we introduce FPD and establish in Section \ref{FPDOT} the connection between OT and FPD. The key conclusions follow in Section \ref{conclusions}. 
\section{Optimal Transport}\label{OT}
Let ($\Omega$, $\mathfrak{F}$, $\mathbbm{P}$) be a probability space. $X$:  $\Omega \mapsto \Omega_{S}$ and $Y$: $\Omega \mapsto \Omega_{T}$ denote two random variables, inducing ($\Omega_{S}$, $\mathfrak{F}_{S}$, $\mu$) and ($\Omega_{T}$, $\mathfrak{F}_{T}$, $\nu$), the source and target measure spaces, respectively.
In the sequel, $\mu$ and $\nu$ denote probability measures, described by their Radon-Nikodym densities \textit{w.r.t} the dominating measure, $\lambda$, which can be instantiated as either the Lebesgue measure or the counting measure. We overload $\mu$ and $\nu$ to denote the probability density functions (pdf) in the continuous case (or probability mass functions (pmf) in the discrete case).\\\\
Optimal Transport was originally introduced by the French mathematician Gaspard Monge to study the problem of shovelling, with minimal cost, a pile of sand into a hole of the same volume \cite{monge1781memoire}. Albeit important, this early formulation was too restrictive as it may not have a solution. Kantorovitch later proposed a probabilistic relaxation, allowing mass to split \cite{https://doi.org/10.48550/arxiv.1803.00567}. In this new setting, the objective is to design a \textbf{transport plan}, that is, a joint pdf $\pi$, satisfying:
\begin{equation}\label{kanto}
\pi_{OT}(x,y|\mathcal{K}) \equiv \argmin_{\pi \in \boldsymbol{\Pi_{\mathcal{K}}}}\left\{\int_{\Omega_{S}\times\Omega_{T}}c(x, y)\pi(x,y)d\lambda(x,y)\right\}
\end{equation}
$c: \Omega_{S}\times\Omega_{T} \mapsto \mathbb{R}^{+} $ is a measurable cost function, $\pi(x,y)$ denotes an unknown (variational) pdf with support in: $\Omega_{S}\times \Omega_{T}$ and $\boldsymbol{\Pi}_{\mathcal{K}}$ denotes the set of joint pdfs $\pi(x,y|\mathcal{K})$ with support in: $\Omega_{S}\times \Omega_{T}$, on which some knowledge constraints $\mathcal{K}$ are imposed. These knowledge constraints correspond to any side information that should be processed in the optimization problem \ref{kanto}. In the context of Optimal Transport, $\mathcal{K}$ represents the marginal constraints that should be satisfied by $\pi$, that is, \begin{equation} \boldsymbol{\Pi}_{\mathcal{K}} \equiv \Bigl\{\pi \in \mathcal{P}(\Omega_{S}\times \Omega_{T}) \;|\; \int_{\Omega_{S}}\pi(x,y)d\lambda(x) = \nu(y) \;\; \text{and} \;\; \int_{\Omega_{T}}\pi(x,y)d\lambda(y) = \mu(x)\Bigr\}\end{equation} 
Here, $\mathcal{P}(\Omega_{S}\times \Omega_{T})$ denotes the set of pdfs in $\Omega_{S}\times \Omega_{T}$.  Interestingly, the objective in \ref{kanto} does not require a parametric model of $\pi$. This is a major distinction between OT and other parametric design methods, mainly copulas \cite{Sklar1973}.  \\\\
In its discrete form, the Kantorovitch problem is a Linear Program (LP) and one would be tempted to apply LP optimization to solve it. However, this program can be prohibitively large, especially in big data regimes. Indeed, if $X$ and $Y$ are discrete random variables, with $\#(\Omega_{S})=n$ and $\#(\Omega_{T})=m$, then the complexity of the LP scales at least in $\mathcal{O}(d^{3}\log(d))$, where $d = max(n,m)$ \cite{https://doi.org/10.48550/arxiv.1610.06447}. A key computationally efficient formulation is based on the idea of entropy regularisation \cite{https://doi.org/10.48550/arxiv.1306.0895}, where the Boltzmann-Shannon entropy of $\pi$ is used to smooth the original problem. \\\\ Towards defining the entropy-regularised OT, let us define the Kullback-Leibler divergence (KLD) from a variational probability density, $\pi$, to a fixed probability density, $\zeta$, as follows: 
\begin{equation}
KL(\pi|| \zeta) \equiv
\left\{
\begin{aligned}
\int_{\Omega_{S}\times\Omega_{T}} \Bigl(\pi(x,y)\log\Bigl(\frac{\pi(x,y)}{\zeta(x,y)}\Bigl)\Bigr)d\lambda(x,y) & \;\; \text{if} \;\; \pi \ll \zeta  \\
+ \infty & \;\; \text{otherwise} \\
\end{aligned} \right.
\end{equation}
The regularised OT problem reads:
\begin{equation}\label{regul_ot}
\pi^{o}_{OT,\epsilon, \phi}(x,y | \mathcal{K}) \equiv \argmin_{\pi \in \boldsymbol{\Pi}_{\mathcal{K}}}\left\{\int_{\Omega_{S} \times \Omega_{T}} c(x, y)\pi(x, y)d\lambda(x,y) + \epsilon KL(\pi||\phi)\right\}
\end{equation}
where $\epsilon > 0$ is the regularisation constant, $\phi$ is a pdf with support in $\Omega_{S}\times\Omega_{T}$ and $\boldsymbol{\Pi}_{\mathcal{K}} \equiv \boldsymbol{\Pi}(\mu, \nu)$. \\\\ If we specialise $\phi$ to the uniform distribution $\mathcal{U}$ with support in $\Omega_{S}\times\Omega_{T}$, then equation \ref{regul_ot} leads to the well-known Boltzmann-Shannon entropy-regularised OT problem \cite{Cuturi_2018}, \cite{https://doi.org/10.48550/arxiv.1610.06447}: 
\begin{equation}\label{ot_entropy}
\pi^{o}_{OT,\epsilon, \mathcal{U}}(x,y | \mathcal{K}) \equiv \argmin_{\pi \in \boldsymbol{\Pi}_{\mathcal{K}}}\left\{\int_{\Omega_{S} \times \Omega_{T}} c(x, y)\pi(x, y)d\lambda(x,y) + \epsilon KL(\pi|| \mathcal{U})\right\}
\end{equation}
It is worth noting that the KLD in \ref{regul_ot} can be generalised to other entropies and more general Bregman divergences, as proposed in \cite{https://doi.org/10.48550/arxiv.1610.06447}.
\\\\
By introducing an entropy term, the Kantorovitch problem becomes strongly convex and efficient iterative scaling algorithms can be used to compute an approximate transport plan (namely, Sinkhorn-Knopp in the discrete case, Fortet in the continuous case) \cite{https://doi.org/10.48550/arxiv.1306.0895}. Though the main reason for introducing an entropy regularizer in OT stems from a computational argument, we argue it can also be viewed as a direct application of Jaynes' Maximum Entropy principle \cite{1056144}. \\\\
In addition to the statistical similarity embedded in the transport cost function $c$, we may need to design a more \textit{structured} joint pdf $\pi(x,y)$. \textit{Structure} can refer to any inherent property in the source and target marginals, such as the semantics of protected attributes (e.g. age, gender, ethnicity), the structure in an image or a graph, etc. It can also refer to a probabilistic knowledge to be processed in the optimization problem \ref{regul_ot}. The existing methodologies addressing structured OT rely either on the notion of sub-modular functions (in particular sub-modular cost functions with diminishing returns) \cite{https://doi.org/10.48550/arxiv.1712.06199} or the addition of a regularisation term (group Lasso, Laplacian) \cite{https://doi.org/10.48550/arxiv.1507.00504} in order to encourage particular mappings over others. \\\\
In the following section, we introduce Fully Probabilistic Design (FPD), which we argue is a generalization of regularised OT. We believe that relaxing the OT formalism into the more generic FPD framework may enable a new set of probabilistic knowledge constraints to be processed in OT.

\section{Fully Probabilistic Design}\label{7}
Fully Probabilistic Design (FPD) is a method for designing pdfs under uncertainty and follows from the principle of minimum discrimination information \cite{10.2307/2528330}. It is, in fact, a generalisation of the classical Bayesian conditioning, allowing the processing of probabilistic knowledge constraints $\mathcal{K}$ in the design of the pdf $\pi(x,y|\mathcal{K})$, without requiring a specified joint model $\pi(x,y, \mathcal{K})$ \cite{QUINN2016532}.\\\\  The axiomatic formulation of FPD, as an optimal distributional design problem, was first established in \cite{KARNY2020104719}, where the authors proved that it is an extension of Bayesian decision making. Later, the FPD framework was extended to Hierarchical Bayesian models in \cite{QUINN2016532}, yielding a stochastic model of the unknown pdf. \\\\
More formally, FPD seeks a pdf $\pi(x,y)$ which satisfies predefined design constraints, formalized as membership of a set, $\boldsymbol{\Pi}_{\mathcal{K}}$, of knowledge-constrained pdfs: \[\pi \in \boldsymbol{\Pi}_{\mathcal{K}}\] 
However, these constraints do not usually specify uniquely the unknown $\pi$. To select the optimal solution, FPD relies on the notion of \textit{ideal design}, which specifies the designer's zero-loss choice of $\pi$. Hence, the ideal design, denoted by $\mathcal{\pi}^{I}$, does not necessarily satisfy the constraints in $\boldsymbol{\Pi}_{\mathcal{K}}$, i.e. $\pi^{I} \notin \boldsymbol{\Pi}_{\mathcal{K}}$. To compute the optimal solution, a utility (loss) function is then used to compare and rank the candidate probability densities, based on their degree of closeness to the ideal design. In \cite{10.1214/aos/1176344689}, the KL divergence is proved to be the most appropriate utility function for ranking preferences in an inference problem. Hence, the solution of the FPD problem is given by: 
\begin{equation}\label{FPD}
   \mathcal{\pi}_{FPD, \pi^{I}}^{o}(x,y|\mathcal{K}) \equiv \argmin_{\pi \in \boldsymbol{\Pi}_{\mathcal{K}}}\left\{KL(\pi(x,y)|| \pi^{I}(x,y)) \right\} \; , \;\; \pi \ll \pi^{I}
\end{equation}

The ideal design in \ref{FPD} is the second argument of the KL divergence and as such, is the zero-KLD datum against which all possible pdfs, consistent with the constraint set $\boldsymbol{\Pi}_{\mathcal{K}}$, are ranked \cite{QUINN2016532}. The FPD solution is therefore the pdf minimizing the KL divergence \textit{w.r.t} the ideal design, while being consistent with $\boldsymbol{\Pi}_{\mathcal{K}}$. \\\\
Given that both the Kantorovitch and the FPD problems seek the optimal design of a joint pdf, it is natural to investigate possible connections between the two frameworks, which we shall study in the next section. 
\section{The FPD-OT framework}\label{FPDOT}
The goal of this section is to formally establish the connection between FPD and OT. We start with the general form of FPD-OT, then derive the special case of the entropy-regularised problem. \newpage 
\subsection{Main theorem}\label{theorem_1}
\begin{theorem}\label{main_theorem}

Let ($\Omega_{S}$, $\mathfrak{F}_{S}$, $\mu$) and  ($\Omega_{T}$, $\mathfrak{F}_{T}$, $\nu$) be two measure spaces, and $\boldsymbol{\Pi}_{\mathcal{K}}$ be the set of joint pdfs $\pi(x,y)$ with support in $\Omega_{S} \times \Omega_{T}$, with prescribed marginals $\mu$ and $\nu$. Let $\phi$ be a fixed pdf, which dominates $\pi$, i.e., $\pi \ll \phi$. $c(x,y) \geq 0$ is a cost function and $\epsilon > 0$ is a regularisation term. \\ Then:
\begin{equation}\label{fpd_problem}
\mathcal{\pi}_{FPD, \pi^{I}}^{o}(x,y|\mathcal{K})  = \pi^{o}_{OT,\epsilon, \phi}(x,y | \mathcal{K}) 
\end{equation}

if $\pi^{I}$, the ideal design, is the extended Gibbs kernel, defined as: \begin{equation} \pi^{I}(x,y) \equiv \exp\Bigl(\frac{-c(x,y)}{\epsilon}\Bigr)\frac{\phi(x,y)}{\mathcal{N}_{\phi, \epsilon}}\end{equation}
Here, $\mathcal{N}_{\phi, \epsilon} \equiv \int_{\Omega_{S} \times \Omega_{T}}\exp\Bigl(\frac{-c(x,y)}{\epsilon}\Bigr)\phi(x,y)d\lambda(x,y)$ is the normalizing constant. 
\end{theorem}
\begin{proof}
From \ref{regul_ot}, we have:
\begin{equation} 
\begin{split}
\pi^{o}_{OT,\epsilon, \phi}(x,y | \mathcal{K})  & \equiv \argmin_{\pi \in \boldsymbol{\Pi}_\mathcal{K}}\int_{\Omega_{S}\times \Omega_{T}}\pi(x,y)\log\left\{\frac{\pi(x,y)\exp(\frac{c(x,y)}{\epsilon})}{\phi(x,y)}\right\}d\lambda(x,y) \\
&= \argmin_{\pi \in \boldsymbol{\Pi}_{\mathcal{K}}}KL(\pi||\pi^{I}) \\
&=  \pi^{o}_{FPD, \pi^{I}}(x,y | \mathcal{K})
\end{split}
\end{equation} 
where the ideal design is of the following form: 
\begin{equation}
\pi^{I}(x,y) \equiv \exp\Bigl(\frac{-c(x,y)}{\epsilon}\Bigr)\frac{\phi(x,y)}{\mathcal{N}_{\phi, \epsilon}}
\end{equation}
\end{proof}
\textbf{Remark 1:} 
When $\epsilon$  $\rightarrow$ $\infty$ , the ideal design is completely defined by $\phi$: \[\pi^{o}_{FPD, \pi^{I}}(x,y|\mathcal{K})  \xrightarrow[\varepsilon \to \infty]{} \pi^{o}_{FPD, \phi}(x,y|\mathcal{K})\]
When $\epsilon$  $\rightarrow$ $0$, we recover, in the limit, the original Kantorovitch problem: 
  \[\pi^{o}_{FPD, \pi^{I}}(x,y|\mathcal{K}) \xrightarrow[\varepsilon \to 0]{} \pi_{OT}(x,y|\mathcal{K})\]

\textbf{Remark 2:} The KL divergence being 1-strongly convex, the problem stated in \ref{fpd_problem} is $\epsilon$-strongly convex, therefore yielding a unique solution. \\\\ 
\textbf{Remark 3:} The ideal design $\pi^{I}$ is composed of two factors: \begin{enumerate*}[label=(\roman*)] \item the Gibbs kernel, which is a function of the transportation cost $c(x,y)$ and \item $\phi$, corresponding to a reference pdf which encodes the designer's preferences about the optimal pdf $\pi^{o}$. \end{enumerate*} \\\\
\textbf{Remark 4:} The Boltzmann-Shannon entropy-regularised problem is an instance of FPD-OT, where the ideal design is the Boltzmann distribution. The energy in this case is the transportation cost $c$ and the regularisation term $\epsilon$ is proportional to the temperature: \\

\begin{corollary}
The entropy-regularised OT problem is a specialization of FPD, where the ideal design $\pi^{I}$ reduces to the Boltzmann distribution: 
\begin{equation}
\mathcal{\pi}_{FPD, \pi^{I}}^{o}(x,y|\mathcal{K})  = \pi^{o}_{OT,\epsilon,\mathcal{U}}(x,y | \mathcal{K}) 
\end{equation}
where
\[\pi^{I}(x,y) \equiv \frac{1}{\mathcal{N}_{\epsilon}}\exp\Bigl(\frac{-c(x,y)}{\epsilon}\Bigr)\]
and $\mathcal{N}_{\epsilon}$ is the normalizing constant of $\pi^{I}$. 
\end{corollary}

\section{Discussion}\label{conclusions}
This brief paper introduced the FPD-OT framework, a generalization of the OT problem using the conventions of FPD. 
\\\\
We believe that the FPD-OT setting can endow the conventional OT framework with the necessary mechanisms for processing more sophisticated, probabilistic knowledge constraints. Indeed, by relaxing the original knowledge-constrained set $\boldsymbol{\Pi}_{\mathcal{K}}$ away from the set $\boldsymbol{\Pi}(\mu, \nu)$, additional moment constraints can be integrated into the optimization problem, that is,  \begin{equation}
\boldsymbol{\Pi}_{\mathcal{K}}\equiv\Bigl\{\pi \in \boldsymbol{\Pi}(\mu, \nu) \; | \; \int_{\Omega_{S}\times\Omega_{T}} \pi(x,y)\chi(x,y)d\lambda(x,y) \leq \eta \Bigr\}
\end{equation}
where $\chi: \Omega_{S}\times\Omega_{T} \mapsto \mathbb{R}^{+}$ is some measurable function and $\eta \geq 0$ is a constant term. 

The second key implication of FPD-OT is the possibility to design robust pdfs, by explicitly modeling the uncertainty in the marginals $\mu$ and $\nu$. This allows $\pi(x,y|\mathcal{K})$ to become a stochastic object, endowed with its own Hierarchical Bayesian model. \\\\ The above consequences of FPD-OT, and their possible applications in machine learning, will be the subject of a future work.

\bibliography{references.bib}
\end{document}